\documentclass[letterpaper, 10 pt, conference]{ieeeconf}  

\IEEEoverridecommandlockouts                              

\usepackage{amsmath}
\usepackage{amssymb}

\usepackage{amsthm}
\usepackage{graphicx}
\usepackage{mathtools}
\usepackage{array}
\usepackage{multicol}
\usepackage{siunitx}
\usepackage{lipsum}
\usepackage[colorlinks=true, allcolors=blue]{hyperref}

\newtheorem{theorem}{Theorem}

\newtheorem{corollary}{Corollary}
\newtheorem{remark}{Remark}

\title{\LARGE \bf
Perception-Integrated Safety Critical Control via Analytic \\ Collision Cone Barrier Functions on 3D Gaussian Splatting
}

\author{Dario~Tscholl,
        Yashwanth~Nakka,
        and~Brian~Gunter
\thanks{*This work was not supported by any organization}
\thanks{All authors are with the Daniel Guggenheim School of Aerospace Engineering at the Georgia Institute of Technology, Atlanta, GA, 30332 USA 
        {\tt\small \{dtscholl3, ynakka3, bgunter8\}@gatech.edu}}%
}

\begin{document}

\maketitle
\thispagestyle{empty}
\pagestyle{empty}

\begin{abstract}
We present a perception-driven safety filter that converts each 3D Gaussian Splat (3DGS) into a closed-form forward collision cone, which in turn yields a first-order control barrier function (CBF) embedded within a quadratic program (QP). By exploiting the analytic geometry of splats, our formulation provides a continuous, closed-form representation of collision constraints that is both simple and computationally efficient. Unlike distance-based CBFs, which tend to activate reactively only when an obstacle is already close, our collision-cone CBF activates proactively, allowing the robot to adjust earlier and thereby produce smoother and safer avoidance maneuvers at lower computational cost.
We validate the method on a large synthetic scene with approximately 170k splats, where our filter reduces planning time by a factor of 3 and significantly decreased trajectory jerk compared to a state-of-the-art 3DGS planner, while maintaining the same level of safety. The approach is entirely analytic, requires no high-order CBF extensions (HOCBFs), and generalizes naturally to robots with physical extent through a principled Minkowski-sum inflation of the splats. These properties make the method broadly applicable to real-time navigation in cluttered, perception-derived extreme environments, including space robotics and satellite systems.
\end{abstract}

\section{INTRODUCTION}
Radiance field representations such as Neural Radiance Fields (NeRFs) \cite{Mildenhall_nerf_2020} and 3D Gaussian Splatting (3DGS) \cite{kerbl_3d_2023} have recently drawn a lot of attention in computer vision. These methods enable high-fidelity novel view synthesis, detailed 3D scene reconstruction and they naturally support tasks such as localization and tracking. A radiance field is a five-dimensional function that maps a 3D point and a viewing direction to an RGB color and an opacity value. NeRFs model this function as a continuous volumetric representation and learn its parameters implicitly through a neural network. Although NeRFs achieve high fidelity reconstructions, they suffer from large training times. 3DGS represents the scene explicitly as a set of Gaussian ellipsoids that capture local geometry and color. Compared to NeRFs, Gaussian splatting allows faster rendering and a more direct geometric representation via Gaussian primitives.

For robotic motion planning, radiance fields offer several advantages over traditional scene representations such as point clouds and occupancy grids. Point clouds and occupancy grids provide only discrete samples and lack inherent information about surface continuity or object connectivity. While signed distance fields (SDFs) can provide such continuity, they are computationally expensive to construct and may cause potential inaccuracies in collision checking \cite{eriksson_fast_nodate}. In contrast, radiance fields inherently encode a continuous description of the environment. This makes them differentiable and allows for a continuous scene interpretation independent of the density of the sensor data.

Recent efforts have begun to exploit 3DGS for motion planning. For example, the authors in \cite{chen_splat-nav_2025} model both robot and splats as ellipsoids and perform collision avoidance via ellipse intersection tests. A similar approach in \cite{chen_safer-splat_2025} derives a control barrier function (CBF) from the shortest Euclidean distance between the robot and the Gaussian splats.
Other work has discretized the world into a grid and computed the grid node cost by using a potential field–like function based on the opacity information of a 3D Gaussian \cite{wang_risk-aware_2024}. A chance-constraint formulation has been proposed by \cite{tagliasacchi_volume_2022} where the authors normalize the Gaussian splats and used a transmittance function to formulate an upper-bound for the probability of collision.

In this letter, we argue that the geometric properties of 3D Gaussian splats have not yet been fully exploited for motion planning, as prior works either only rely on a distance-based formulation or employ a fully probabilistic approach.
Our key observation is that given the quadratic shape of 3D Gaussian splats, we can convert each Gaussian into a closed-form forward collision cone. A collision cone characterizes a set of relative velocities that would ultimately lead to collision. Taking the complement of the collision cone yields a control barrier function (CBF) that defines a collision free set, hence provides formal safety guarantees. Compared to distance-based CBFs, collision cone CBFs do not require high-order extensions (HOCBF) \cite{xiao_control_2019}, i.e. they are computationally inexpensive, and exhibit greater robustness in dynamic environments \cite{tayal_control_2024}. Moreover, because collision cone constraints activate proactively as the robot approaches an obstacle, they enable smoother obstacle avoidance than distance-based methods, which tend to be more reactive.

Beyond efficiency, a 3DGS-based formulation eliminates the need for conservative obstacle inflation. In traditional collision cone applications, obstacles are approximated using simple geometric shapes, such as ellipsoids or spheroids, because a closed-form cone cannot generally be derived for arbitrary non-quadratic shapes \cite{chakravarthy_generalization_2012, goswami_collision_2024}. Therefore, 3DGS enables both safe and efficient motion planning.
%
The main contributions of this letter are summarized as follows:
\begin{itemize}
    \item \textit{Analytical 3DGS-based collision cone formulation:} This work introduces a collision cone–based control barrier function derived directly from 3D Gaussian Splatting, enabling exact analytical cones without conservative obstacle inflation or approximations.
    \item \textit{Computationally-efficient safety verification:} Compared to existing planners, the proposed collision cone CBF is computationally inexpensive due to its first-order nature, i.e. no need for high-order CBF extensions. Additionally, the collision cone formulation enables proactive constraint activation and smooth obstacle avoidance.
\end{itemize}

\section{PRELIMINARIES}
    \subsection{Problem Statement}
    In this letter we consider the navigation problem from an initial configuration $x(0) = x_0$ to a specified goal configuration $x(T) = x_f$ such that there are no collisions at any time $t =[0,T]$. For this, we assume a system that can be represented in the form of a nonlinear affine control system
    \begin{equation} \label{eq_affine_control_system}
        \dot{x} = f(x) + g(x)u, \qquad x(0) = x_0,
    \end{equation}
    where $x \in \mathcal{X} \subset \mathbb{R}^n$ represents the state of the system and $u \in \mathcal{U} \subset \mathbb{R}^m$ is the control input. The functions $f: \mathbb{R}^n \rightarrow \mathbb{R}^n$ and $g : \mathbb{R}^n \rightarrow \mathbb{R}^{n \times m}$ are continuous and locally Lipschitz. For simplicity and computational efficiency, a drone represented as a double-integrator model will be considered throughout this letter, however, the proposed method is valid for all systems that satisfy (\ref{eq_affine_control_system}). 
    %
    %
    \subsection{Safety and Control}
    To formalize safety, we introduce Barrier functions  \cite{xiao_control_2019, lindemann_control_2019}. A continuously differentiable function $h:\mathbb{R}^n \rightarrow \mathbb{R}$ is a barrier function for system (\ref{eq_affine_control_system}) if there exists a class-$\mathcal{K}$ function $\alpha$ such that
    \begin{equation} \label{eq_barrier_function}
        \dot{h}(x) + \alpha \bigl( h(x) \bigl) \geq 0 \qquad \text{for all} \, x \in \mathcal{C}.
    \end{equation}
    A continuous function $\alpha : [0,a) \rightarrow [0, \infty) \, \text{with} \, a > 0$ is said to belong to class-$\mathcal{K}$ if it is strictly increasing and $\alpha(0) = 0.$ \cite{xiao_control_2019}. Barrier functions were first used in optimization and are now commonly used together with Lyapunov-like functions to establish safety, avoidance or eventuality properties \cite{ames_control_2016}. The natural extension of a barrier function to a system with control inputs is a Control Barrier Function (CBF) \cite{wieland_constructive_2007, zeng_safety-critical_2021, ames_control_2019}. For safety-critical control, we consider a safe set $\mathcal{C}$ defined as the superlevel set of a continuously differential function $h:x \in \mathcal{X}$:
    \begin{equation} \label{eq_safe_set}
        \mathcal{C} = \{x \in \mathcal{X} \, | \, h(x) \geq 0\}.
    \end{equation}
    The function $h$ is a control barrier function if $\partial h/ \partial x \neq 0$ for all $x  \in \partial C$ and there exists a class-$\mathcal{K}$ function $\alpha$ such that for the control system (\ref{eq_affine_control_system}) $h$ satisfies
    \begin{equation} \label{eq_cbf_def}
          \underset{u \in \mathcal{U}}{\text{sup}} \,[L_fh(x) + L_gh(x)u]  \geq -\alpha \bigl( h(x) \bigl) \quad \text{for all} \, x \in \mathcal{C}.
    \end{equation}
    This definition allows us to consider the set of all stabilizing controllers $K_{cbf}$ for every point $x \in  \mathcal{X} $
    \begin{equation} \label{eq_safe_control_set}
            K_{cbf}(x) := \{ u \in \mathcal{U} \, | \, L_fh(x) + L_gh(x)u + \alpha(h(x)) \geq 0 \},
    \end{equation}
    where, for a control system given in (\ref{eq_affine_control_system}), any locally Lipschitz controller $u \in K_{cbf}(x)$ renders $\mathcal{C}$ forward invariant which implies that the control system is safe \cite{ames_control_2019}.
    Aside from safety guarantees, CBFs are attractive for safety-critical systems because the constraint formulation is affine in the control input, allowing fast real-time computation. CBFs are often used in combination with quadratic programs (QPs) to synthesize safe controllers for robotic systems. The quadratic program for a double-integrator system subject to the CBF constraint in (\ref{eq_cbf_def}) is given by
    \begin{equation}
        \begin{aligned}
            \min_{u} &\quad \|u - \bar{u}\|_2^2 \\
            \textrm{s.t.} &\quad L_fh(x) + L_gh(x)u  \geq -\alpha \bigl( h(x) \bigl)
        \end{aligned}
    \end{equation}
    where $\bar{u}$ denotes a control reference, $L_fh$ and $L_gh$ are the Lie derivatives with respect to the $f$ and $g$ functions, and $\alpha$ is a class-$\mathcal{K}$ function. 
    %
    %
    \subsection{3D Gaussian Splatting}
    3D Gaussian Splatting (3DGS) represents a scene as a set of anisotropic Gaussian primitives. Each primitive is parameterized by a mean position $\mu \in \mathbb{R}^3$, a covariance matrix $\Sigma \in \mathbb{S}_{++}$, an opacity $\alpha \in [0,1]$, and spherical harmonic coefficients encoding view-dependent color. The Gaussian density is defined as $G(x) := exp\bigl( -\frac{1}{2}(x)^\top \Sigma^{-1}(x)\bigl)$, which corresponds to an ellipsoidal volumetric kernel centered at $\mu$ used for the blending/rendering process. Unlike a traditional covariance matrix, the covariance matrix used in 3DGS describes the spatial configuration of a splat and is given by 
    \begin{equation} \label{eq_cov_matrix}
        \Sigma = RSS^\top R^\top,
    \end{equation}
    where $R \in \mathrm{SO}(3)$ is a rotation matrix defined by quaternions and $S=diag(s_1, s_2, s_3)$ with $s_i > 0$ is a diagonal scaling matrix. During optimization, both geometric $(\mu,\Sigma)$ and photometric (opacity, spherical harmonics) parameters are updated via gradient descent. An adaptive density control mechanism inserts, clones, or splits Gaussians to refine under- and over-reconstructed regions, yielding a compact yet expressive representation of the scene. For rendering, each Gaussian is projected into the image plane, where anisotropic splats are rasterized and composited using visibility-aware $\alpha$-blending.

    In this letter, we want to use 3DGS to encode collision avoidance/safety constraints into the trajectory tracking problem. Therefore, for collision avoidance, we only care about the geometric properties of a splat. The geometry of a 3D Gaussian splat can be represented by its confidence ellipsoid
    \begin{equation} \label{eq_gaussian_ellipsoid}
        \mathcal{E} = \{x \in \mathbb{R}^3 \, | \, (x-\mu)^\top\Sigma^{-1}(x-\mu) \leq c^2\},
    \end{equation}
    where $c^2 = \chi^2_{3,0.99}$ represents the 99th percentile of the chi-squared distribution with 3 degrees of freedom. 
    %
    %
\section{MAIN RESULTS}
    %
    \subsection{3DGS-Based Collision Cone}
    \begin{theorem}[Collision cone for a 3D Gaussian splat] \label{theorem_cc}
        Let $\Sigma \in \mathbb{S}_{++}^3$ be positive definite and set $A=\Sigma^{-1}$. 
        For $c>0$, define the Gaussian splat confidence ellipsoid
        \[
        \mathcal{E} = \bigl\{x \in \mathbb{R}^3 \;\big|\; (x-\mu)^\top A (x-\mu) \le c^2 \bigr\}.
        \]
        Let a point robot be located at $p \in \mathbb{R}^3$ and move with constant velocity 
        $v \in \mathbb{R}^3$ along the ray $x(t)=p+t v$, $t \ge 0$. 
        Define the line-of-sight vector $r := \mu - p$, and assume the robot starts outside the ellipsoid, i.e.\ $r^\top A r > c^2$. Then the following statements are equivalent:
        \begin{enumerate}
            \item[\emph{(i)}] There exists $t \ge 0$ such that $x(t)\in\mathcal{E}$.
            \item[\emph{(ii)}] The velocity $v$ satisfies the forward collision cone conditions
            \begin{subequations}\label{eq:cone_conditions}
            \begin{align}
               (v^\top A v)(r^\top A r - c^2) - (r^\top A v)^2 &\le 0, \label{eq:cone_main}\\
               r^\top A v &\ge 0. \label{eq:cone_approach}
            \end{align}
            \end{subequations}
        \end{enumerate}
    \end{theorem}
    \begin{proof}[Proof via quadratic minimization]
        Substituting the ray $x(t)=p+t v$ into the ellipsoid inequality gives the quadratic
        \begin{equation} \label{eq_ray-ellipsoid_problem}
            \phi(t) = (r-t v)^\top A (r-t v) - c^2 = a t^2 - 2 b t + d,
        \end{equation}
        where $a = v^\top A v >0$, $b = r^\top A v$, and $d = r^\top A r - c^2 >0$.
        Thus, a collision occurs iff there exists $t \ge 0$ such that $\phi(t)\le 0$.
        
        Since $\phi$ is convex, its unique minimizer is
        \[
        t^\star = \frac{b}{a},
        \]
        with minimum value
        \[
        \phi_{\min} = d - \frac{b^2}{a}
        = \frac{(v^\top A v)(r^\top A r - c^2) - (r^\top A v)^2}{v^\top A v}.
        \]
        Necessity: If $\exists\, t_0\ge 0$ with $\phi(t_0)\le 0$, then $\phi_{\min}\le 0$, which yields \eqref{eq:cone_main}. Furthermore, feasibility requires $t^\star \ge 0$, i.e.\ $r^\top A v \ge 0$, which is \eqref{eq:cone_approach}.
        Sufficiency: If \eqref{eq:cone_main} and \eqref{eq:cone_approach} hold, then $t^\star \ge 0$ and $\phi_{\min}\le 0$, hence $\phi(t^\star)\le 0$, proving a collision. 
    \end{proof}
    \begin{proof}[Alternative proof via whitening]
        To show equivalence with classical collision cones, consider the factorization $\Sigma = R S S^\top R^\top$ where $R \in \mathrm{SO}(3)$ and $S=\mathrm{diag}(s_1,s_2,s_3)$. Define the whitening matrix
        \begin{equation} \label{eq_whiten_matrix}
            L := S^{-1} R^\top,
        \end{equation}
        and the whitened variables $\tilde r := L r$, $\tilde v := L v$. The ellipsoid $\mathcal E$ maps to the sphere
        \[
        \tilde{\mathcal E} = \{ y \in \mathbb R^3 \mid \|y\|^2 \le c^2 \}.
        \]
        The collision condition becomes
        \[
        \phi(\lambda) = \|\tilde r - \lambda \tilde v\|^2 - c^2 \le 0,\quad \exists \lambda \ge 0.
        \]
        This expands to the classical Lorentz-cone inequalities
        \begin{subequations}\label{eq:lorentz}
        \begin{align}
           \|\tilde v\|^2 (\|\tilde r\|^2 - c^2) - (\tilde r^\top \tilde v)^2 &\le 0, \\
           \tilde r^\top \tilde v &\ge 0.
        \end{align}
        \end{subequations}
        Writing $\tilde v$ in spherical coordinates yields
        \[
        \|\tilde r\|^2 (V_\theta^2 + V_\phi^2) \le c^2 (V_\theta^2 + V_\phi^2 + V_r^2), \quad V_r < 0,
        \]
        which are exactly the generalized 3D collision cone conditions of Chakravarthy \& Ghose \cite{chakravarthy_generalization_2012}. Since $L$ is nonsingular, inequalities \eqref{eq:lorentz} are equivalent to \eqref{eq:cone_conditions} in the original coordinates. Hence the collision cone derived for ellipsoids is consistent with the known analytical conditions for a point-sphere system.
    \end{proof}
    Conceptually, (\ref{eq:cone_main}) represents a pair of opposite cones in velocity space with (\ref{eq:cone_approach}) disregarding the cone that points away from the obstacle, leaving a single, forward collision cone. An illustration of the described concept is shown in Figure \ref{fig_collision_cone}.

    \begin{figure}[htp]
        \includegraphics[width=\columnwidth] {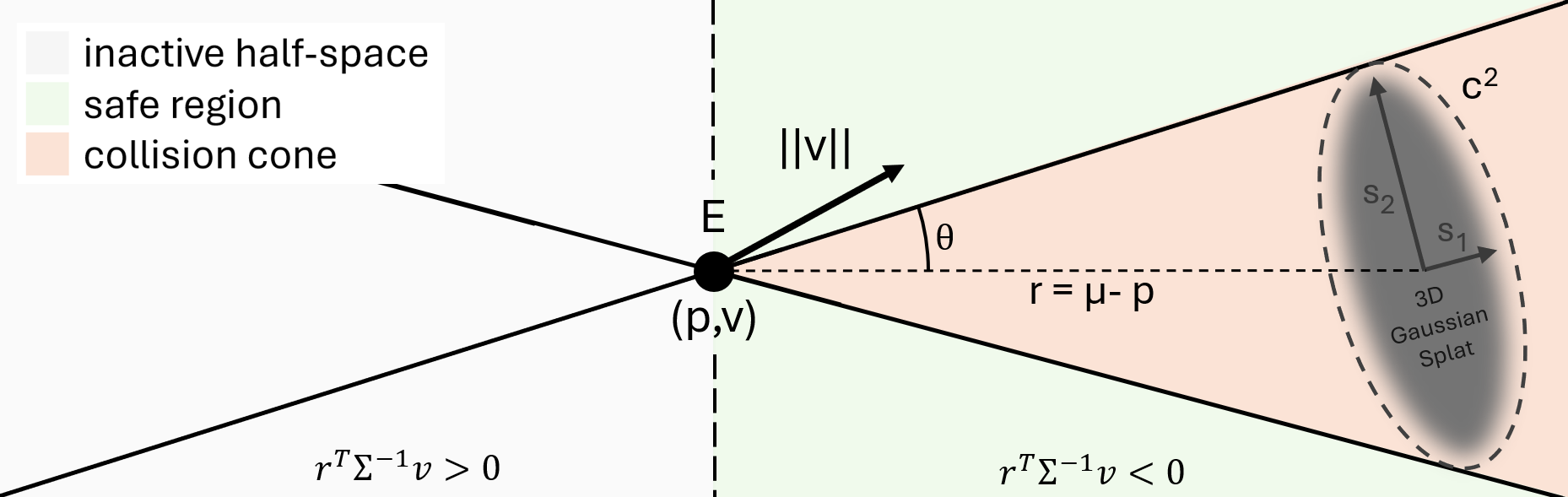} 
        \caption{2D illustration of a Gaussian splat collision cone. A robot $E$, located at $p$ with velocity $v$, forms a collision cone with a 3D Gaussian splat located at $\mu$. The shape of the 3D Gaussian splat is given by the elements of the diagonal scaling matrix $S$ and the confidence level denoted by $c^2$.}
        \label{fig_collision_cone}
    \end{figure}
    \begin{figure*}[htp] 
        \includegraphics[width=\textwidth,height=5.8cm] {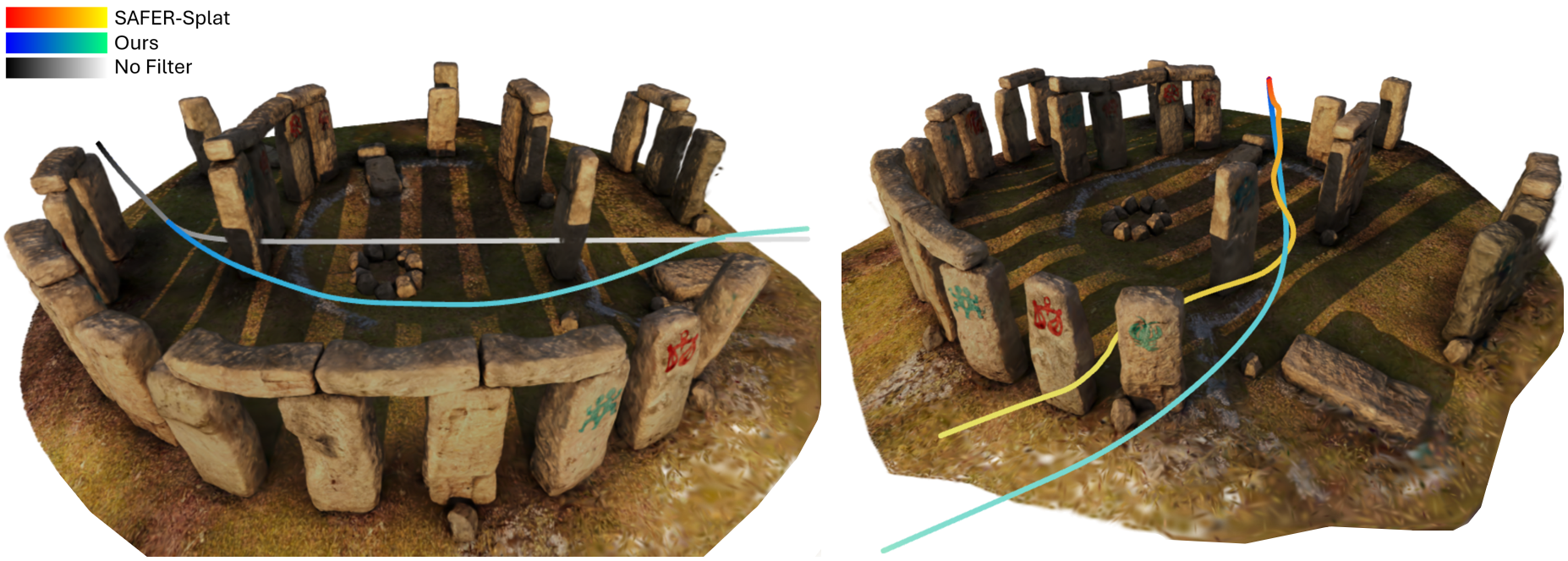} 
        \caption{Qualitative comparison between SAFER-Splat and our proposed method. The left image shows two trajectories with the same start and goal configurations. For the gray trajectory that is controlled by the PD controller, the safety filter was turned off. Ours had the safety filter turned on. It can be seen that the robot is able to safely and smoothly avoid any obstacles. The image on the right shows a comparison between our method and SAFER-Splat. It can be seen that our method is less reactive to obstacles due to proactive constraint activation and thus generates a smoother trajectory.}
        \label{fig_quali_difference}
    \end{figure*}
    %
    \subsection{Collision Cone CBF} \label{sec_cbf_qp}
    Since the complement of the collision cone condition defines the admissible safe set, it yields a natural candidate for a control barrier function, expressed as
    \begin{equation} \label{eq_cbf_candidate}
        h(p,v) = \underbrace{(v^\top Av)}_\beta \underbrace{(r^\top Ar - c^2)}_\gamma - \underbrace{(r^\top Av)^2}_{\delta^2} \geq 0,
    \end{equation}
    with the safe set $\mathcal{C}_i := \{(p,v) : h_i(p,v) \geq 0 \}$ defined for the $i$th splat. The scalar variables $\beta$, $\gamma$ and $\delta$ are introduced to simplify notation.
    \begin{corollary}
        Let $\Sigma \in \mathbb{S}^3_{++}$ be positive definite and $c>0$. For a point robot at $p\in\mathbb{R}^3$ with constant velocity $v\in\mathbb{R}^3$ and initially outside the ellipsoid, i.e. $r^\top A r>c^2$, set $r:=\mu-p$.
        Define the safe set of a splat
        \[
        \mathcal{C}_i = \{(x,p) \in \mathcal{X} \, | \, h_i(p,v) \geq 0\}        .    
        \]
        Then the corresponding control barrier function candidate
        \[
        h_i(p,v) := (v^\top Av)(r^\top Ar - c^2) - (r^\top Av)^2 \geq 0,
        \]
        is exactly the complement of the forward collision cone of Theorem \ref{theorem_cc}.
    \end{corollary}
    \begin{remark}[Global Safe Set in a Scene]
        So far, only one Gaussian splat has been considered, but in a scene with thousands of Gaussian splats, the global safe set is simply given by the intersection $\mathcal{C}:=\bigcap\limits_{i \in \mathcal{I}} \mathcal{C}_i$.
    \end{remark}
    \begin{remark}[Dynamic Environments]
        The proposed CBF is derived under the assumption that both robot and obstacle velocities remain constant within a control loop. In our implementation, we consider static obstacles by setting their velocity to zero. However, in dynamic environments, obstacle velocities can be modeled as piecewise-continuous over the control interval $\Delta t$, ensuring that the CBF formulation remains valid. In this case, the obstacle mean $\mu$ becomes time-dependent, and the relative velocity with respect to the robot at time step $k$ is given by $v_k = \dot{\mu}_k - \dot{p}_k$.
    \end{remark}
    Since (\ref{eq_cbf_candidate}) is a function of position and velocity, a standard first-order CBF as defined in (\ref{eq_cbf_def}) is sufficient assuming the control input only affects the acceleration term. Let $w(x) := \gamma Av - \delta Ar \in \mathbb{R}^3$ with $\gamma$ and $\delta$ from (\ref{eq_cbf_candidate}), the Lie derivatives in (\ref{eq_cbf_def}) become
    \begin{equation}
        \begin{aligned} \label{eq_lie_derivatives}
            L_f h(p,v) &= 2 w(x)^\top f_v(x), \\
            L_gh(p,v)u &= 2 w(x)^\top g_v(x)u, 
        \end{aligned}
    \end{equation}
    where $f_v$ and $g_v$ stem from the control affine representation of the velocity dynamics in (\ref{eq_affine_control_system}), expressed as $\dot{v}=f_v(x) +g_v(x)u$. Note that $\nabla_ph^\top \dot{p} = 2 ( \beta Ar - \delta Av)v = 0$ and therefore doesn't show up in (\ref{eq_lie_derivatives}). With a linear class-$\mathcal{K}$ function $\alpha (h) = p_k h$ with $p_k >0$, the full CBF constraint can be written as
    \begin{equation} \label{eq_cbf_cc}
         \bigl( g_v(x)^\top w(x) \bigr)^\top u \geq - \frac{p_k}{2} h(p,v)- w(x)^\top f_v(x).
    \end{equation}
    With $a(x) = g_v(x)^\top w(x)$ and $b(x) = - \frac{p_k}{2} h(p,v)- w(x)^\top f_v(x)$, the constraint in (\ref{eq_cbf_cc}) is feasible at $x$ iff the supremum of $a(x)^\top u$ over a convex compact $\mathcal{U}$ is at least the threshold $b(x)$, hence iff the support function $\sigma_\mathcal{U} := sup_{u\in \mathcal{U}} \,a^\top u$ satisfies $\sigma_\mathcal{U}\bigl (a(x) \bigl)  \geq b(x)$, which is equivalent to (\ref{eq_cbf_def}).
    %
    For a more detailed discussion on the influence of choosing a different class-$\mathcal{K}$ function, we refer the reader to \cite{xiao_control_2019}.
    %
    %
    \begin{remark}[Physical Extent of the Robot] \label{ref_dimens}
        Up until this point, the robot has been assumed to be a point object. One way to consider the physical extent of the robot is to vary the level of uncertainty by increasing $c^2$, as larger values inflate the volume of the confidence ellipsoid associated with each Gaussian. 
        
        Because it is difficult to correlate an uncertainty level to a volumetric variation, another way to account for the robot's size is to model the robot as a sphere with Euclidean radius $\rho$. The ray-ellipsoid collision check from (\ref{eq_ray-ellipsoid_problem}) then becomes the intersection of the ray with the Minkowski sum of the confidence ellipsoid and the robot sphere. 
        Using the whitening matrix $L$ defined in (\ref{eq_whiten_matrix}), the confidence ellipsoid $\mathcal{E}_c$ becomes a sphere $L(\mathcal{E}_c) = \mathcal{S}_c$. In whitened space, the robot sphere $\mathcal{S}_r$ becomes an ellipsoid $\mathcal{E}_r := \{y \enspace | \enspace y^\top (S^\top S) y \leq \rho^2 \}$. Therefore, the robot's radial function along a unit direction $n \in \mathbb{S}^2$ (in whitened coordinates) is $r_{\mathcal{E}_r} = \rho/\sqrt{n^\top(S^\top S)n}$. The Minkowski sum of $r_{\mathcal{E}_r}$ and $r_{\mathcal{S}_c}$ can then be written as $c_{M} =  r_{\mathcal{E}_r} \oplus r_{\mathcal{S}_c} = c + \rho/\sqrt{n^\top(S^\top S)n}$, where $c$ is the radius of the confidence sphere $\mathcal{S}_c$. As can be seen, $c_M$ is direction dependent. The relevant direction in whitened space is the unit vector $n$ orthogonal to the velocity and pointing from the ray to the splat center. Hence, $n = \bar{n}/\|\bar{n}\|$ where $\bar{n} = (I-\tilde{v}\tilde{v}^\top)\tilde{r}$. While this is exact, sometimes it is enough to use a simpler, more conservative approach where $n$ is set constant, in which case $\mathcal{E}_r$ becomes a subset of a sphere with radius $r_{\mathcal{E}_r} = \rho/s_{min}$. 

        Therefore, when considering the physical extent of the robot, $c^2$ in (\ref{eq_cbf_candidate}) can simply be replaced with $c_M^2$
        \[
            (v^\top Av )(r^\top Ar - c_M^2) - (r^\top Av)^2 \leq 0.
        \]
        If one wants to express everything in the original coordinates, a transformation of $c_M$ is necessary. The transformation of $c_M$ to the original coordinates is provided in the appendix. 
    \end{remark}
%
%
\section{SIMULATION SETUP AND RESULTS} \label{sec_simulation}
In this section, we compare our method to SAFER-Splat \cite{chen_safer-splat_2025}, a state-of-the-art 3DGS-based planner that is the most similar to ours. For SAFER-Splat, the authors designed a CBF that minimizes the distance between the robot position $p$ and some point $y*$ on the ellipsoid. By solving a robot-ellipsoid distance program they derived closed-form gradients and Hessians of this distance, and incorporated the resulting CBF in a quadratic program. In the following, we provide a qualitative and quantitative comparison between SAFER-Splat and our proposed method.

\textit{Setup:} 
For the control problem, SAFER-Splat assumed an input affine control system as given in (\ref{eq_affine_control_system}). For the robot, they chose to use a quadcopter that was modeled as a double-integrator system, as shown in (\ref{eq_double_int}).
\begin{equation} \label{eq_double_int}
    \dot{x} = 
    \begin{bmatrix}
        \dot{p} \\
        \dot{v}
    \end{bmatrix} =
    \begin{bmatrix}
        0 & I \\
        0 & 0
    \end{bmatrix}
    \begin{bmatrix}
        p \\
        v
    \end{bmatrix} + 
    \begin{bmatrix}
        0 \\
        I
    \end{bmatrix} u,
\end{equation}
where $I$ denotes the identity matrix. Based on this state space, the $f_v(x)$ and $g_v(x)$ functions for our CBF formulation are $0$ and $I$, respectively. Plugging these values into (\ref{eq_cbf_cc}) yields the collision cone CBF for a double-integrator system given by
\begin{equation} \label{eq_double_int_cbf}
    w(x)^\top u
    \geq -\frac{p_k}{2} h(p,v).
\end{equation}
For the control input, a simple PD controller $\bar{u}$ was used to navigate the drone to the goal point. In terms of control constraints, the QP has been extended to enforce a Euclidean bound on the velocity and acceleration, i.e. $\|u\| \leq a_{max}$, without compromising the feasibility or safety of the program. The formal proof for this is given in \cite{chen_safer-splat_2025} under Corollary 1.
%
%

\textit{Implementation:} 
The QP was implemented using Clarabel \cite{clarabel_2024}, an interior point numerical solver for convex optimization problems. The constraint evaluations were performed on GPU using PyTorch, and the simulations were run on a computer with an NVIDIA RTX 3050. The map used for the simulation was a synthetic data set we refer to as Stonehenge. The Stonehenge scene was trained using Nerfstudio \cite{nerfstudio} and specifically Splatfacto, which is Nerfstudio's implementation of 3D Gaussian splatting based on the original 3DGS paper \cite{kerbl_3d_2023}. To avoid highly anisotropic Gaussians, a scale regularizer \cite{xie_physgaussian_2023} and a scale adjustments were used during training. Highly anisotropic Gaussians can cause numerical issues, like floating point inaccuracies, because eigenvalues of $\Sigma^{-1}$ can easily take on values in the order of $10^{16}$ if left unconstrained. After training, our Stonehenge map consisted of around 170k Gaussians, with the largest eigenvalues in the order of $10^{8}$.

\textit{Experiment I:}
For the first experiment, shown in Figure \ref{fig_quali_difference}, we provide a qualitative comparison between (1) the default PD controller with and without our filter and (2) the performance of SAFER-Splat's filter versus ours. Note that in both cases, the robot was assumed to be a point object. Also, an adaptive filter to only consider Gaussian splats within a certain distance of the robot had been implemented. The image on the left (1) shows that with our proposed filter, the robot is able to avoid the obstacles and generate a safe and smooth trajectory. 
The image on the right (2) demonstrates the advantage of the collision cone-based CBF relative to a distance-based formulation. With constraints activating proactively, our approach yields smoother trajectories than those obtained with SAFER-Splat.

\textit{Experiment II:}
For the second experiment, we provide a quantitative comparison to be able to compare the efficiency, hence planning time and jerk of both methods. Like before, the point robot assumption and adaptive filter were used. This time, 50 trajectories were simulated with starting points spread evenly around the map, and goal points on the opposite sides (\SI{180}{\degree} shifted). The top image of Figure \ref{fig_qp_safer_cc} shows the planning time per trajectory plus the overall average for both SAFER-Splat and our method. It can be seen that our filter reduces planning time by a factor of 3 without compromising on success rate or safety. The bottom image of Figure \ref{fig_qp_safer_cc} shows a jerk/smoothness distribution over the generated trajectories. The jerk has been quantified in three different ways; normalized jerk (nJ) with respect to the traveled distance, the root mean square of the jerk (RMS J) and the classic minimum-jerk objective $\int\|j\|^2dt$, representing the integrated squared jerk (ISJ). The plots suggest that the trajectories of our method are smoother but also shorter which is what can be seen in the normalized jerk plot. The shortness of the trajectories is due to the proactive constraint activation from the collision cones as they terminate earlier if infeasible. The integrated jerk of our method is slightly lower but shows a heavier tail reinforcing that the overall jerk energy is lower but also indicating sharp corrections, which we could observe at the goal points.
\begin{figure}[htp]
    \includegraphics[width=\columnwidth] {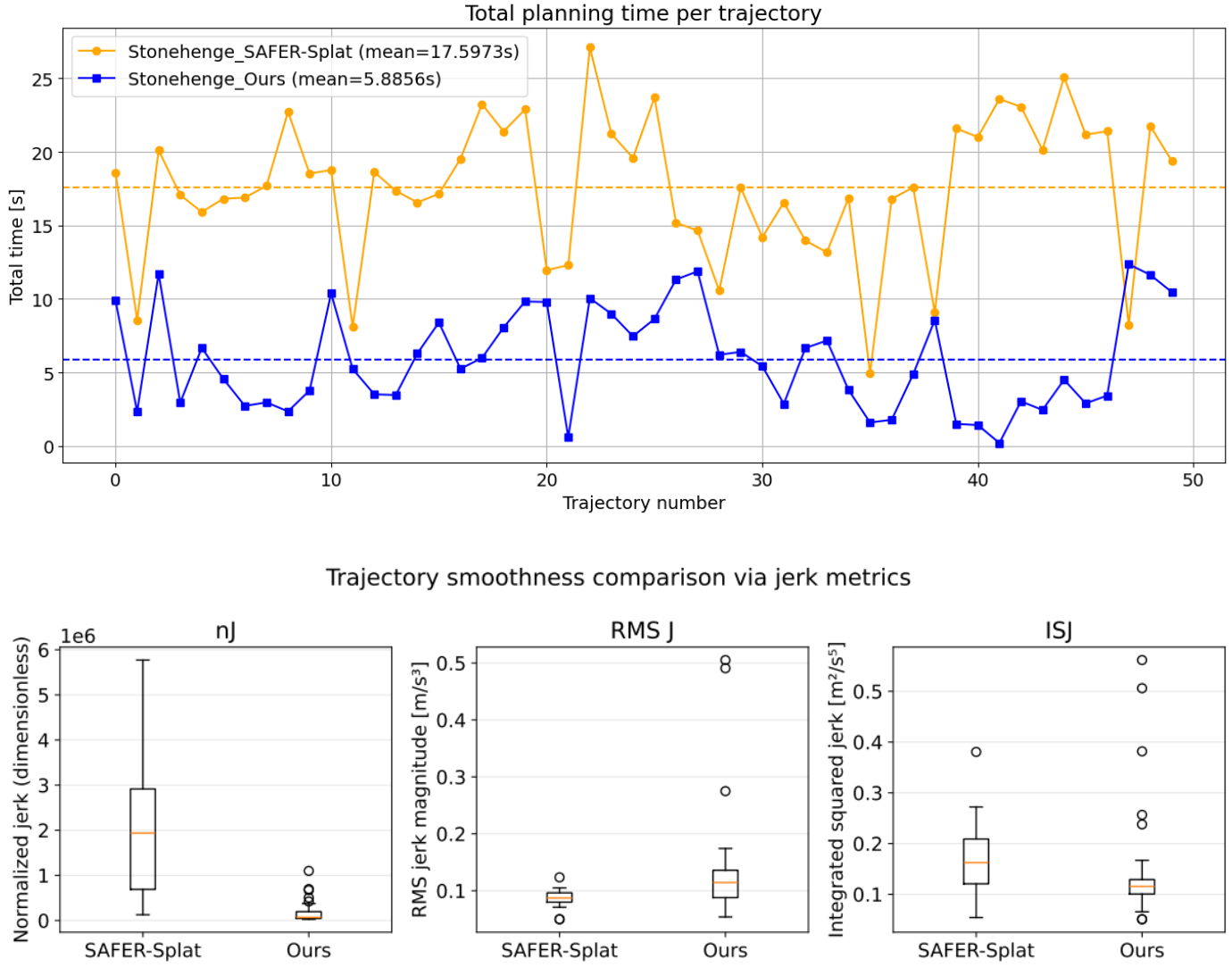} 
    \caption{Efficiency comparison between our method and SAFER-Splat. A) shows that our method reduces planning times by a factor of 3. B) indicates that our approach generates less jerk but performs sharper corrections. The approach condition divides the space into an active and inactive half-space.}
    \label{fig_qp_safer_cc}
\end{figure}
\textit{Observations:} When comparing our method to SAFER-Splat we saw that the success rates were almost identical across different trajectory sample sizes. One difference that we noticed is that when the start configuration is initialized very close to an obstacle, the robot with our filter does not have enough control authority, hence lacks the angular velocity needed, to escape the collision cone within one time step. This causes $h$ to be negative and thus the problem becomes infeasible because the CBF constraints in the QP are implemented as hard constraints. While this behavior can be fixed by initializing the robot further away from an obstacle or by introducing slack variables, SAFER-Splat performed better in that regard as it was able to avoid obstacles in such cases.
\section{CONCLUSIONS}
This letter presented a perception‑driven safety filter that converts a 3D Gaussian Splatting (3DGS) ellipsoid into an analytic collision cone control barrier function (CBF) and embeds it within a QP framework. 
By leveraging the geometric properties of a 3D Gaussian splat, we formulated a ray–ellipsoid intersection problem through which we obtained necessary and sufficient conditions for a forward collision cone. The complement defines an analytical-save filter that activates proactively as the robot approaches an obstacle. Because of the first-order nature of the CBF, the filter is low-cost, making it suitable for real-time implementation.
We demonstrated in simulation that the collision cone-based filter enables collision-free trajectories in a cluttered environment. We compared it to a state-of-the-art 3DGS planner, where we demonstrated that our method generates smoother trajectories while reducing planning time by a factor of 3 and maintaining safety guarantees. In the future, we plan to incorporate slack variables and integrate our filter into a model predictive control (MPC)-based framework for better constraint handling.

\addtolength{\textheight}{-12cm}   



\section*{APPENDIX} \label{app_minkowski}
In (\ref{ref_dimens}) we showed that the  Minkowski sum of $r_{\mathcal{E}_r}$ and $r_{\mathcal{S}_c}$ was $c_{M} =  r_{\mathcal{E}_r} \oplus r_{\mathcal{S}_c} = c + d/\sqrt{n^\top(S^\top S)n}$. To represent $c_M$ in the original coordinates, we introduce the following notation.

Besides $\beta = v^\top Av$ and $\delta = r^\top Av$ we define the A-orthogonal projection of $r$ onto $v^\perp$ given by $t=r-v \frac{\delta}{\beta}$. Hence, the robot's radial size in the original coordinates becomes $\psi(p,v) = \frac{\sqrt{t^\top At}}{\|t\|_2}$. Finally, $c_M(p,v) = c + \rho\psi(p,v)$, where $\rho$ is the Euclidean robot sphere radius.
The simpler, more conservative expression doesn't change with the transformation, thus, in that case, $c_M = c + \frac{\rho}{s_{min}}$.
The robot-inflated version of (\ref{eq_cbf_candidate}) therefore becomes:
\[
    h(p,v) = \underbrace{(v^\top Av)}_\beta \underbrace{(r^\top Ar - c_M^2)}_\eta - \underbrace{(r^\top Av)^2}_{\delta^2} \geq 0
\]
The Lie derivatives for this case can then be expressed as
\begin{align*}
    L_fh =& \enspace 2\Bigl( (\eta Av - \delta Ar)f_v(x) \\
    &- \beta c_M \bigl( (\nabla_p c_M)^\top f_p(x) + (\nabla_v c_M)^\top f_v(x) \bigl) \Bigl), \\
     L_gh = & \enspace 2\bigl( \eta Av - \delta Ar -\beta c_M(\nabla_v c_M)^\top \bigl) g_v(x).
\end{align*}
where the gradients of $c_M$ are given by
\[
    \nabla_p c_M = \rho\frac{\partial t}{\partial r}^\top \nabla_t \psi,
    \qquad \nabla_v c_M = \rho\frac{\partial t}{\partial v}^\top \nabla_t \psi
\]
with $\nabla_t \psi$ and the partial derivatives of $t$ as follows
\[
    \nabla_t \psi = \frac{At}{\|t\| \sqrt{t^\top At}} - \frac{\sqrt{t^\top At} }{\|t\|^3}t,
\]
\[
    \frac{\partial t}{\partial r} = I - \frac{v(Av)^\top}{\beta},
\]
\[
    \frac{\partial t}{\partial v} = -\frac{v(\beta Ar - 2\delta Av)^\top}{\beta^2} -\frac{\delta}{\beta}I.
\]
As before, these derivations assume that the control input only affects the linear acceleration term, but else are valid for any control system in the form of \ref{eq_affine_control_system}.


\bibliographystyle{IEEEtran}
\bibliography{root}

\end{document}